\documentclass{article}

     \PassOptionsToPackage{numbers, compress}{natbib}
  \usepackage[preprint]{neurips_2026}

\usepackage[utf8]{inputenc} 
\usepackage[T1]{fontenc}    
\usepackage{hyperref}       
\usepackage{url}            
\usepackage{booktabs}       
\usepackage{amsfonts}       
\usepackage{nicefrac}       
\usepackage{microtype}      
\usepackage{xcolor}         

\usepackage{amsmath}
\usepackage{amssymb}
\usepackage{mathtools}
\usepackage{amsthm}

\usepackage{amsmath,amsfonts,bm}

\def\eqref#1{equation~\ref{#1}}

\def\1{\bm{1}}

\DeclareMathAlphabet{\mathsfit}{\encodingdefault}{\sfdefault}{m}{sl}
\SetMathAlphabet{\mathsfit}{bold}{\encodingdefault}{\sfdefault}{bx}{n}

\def\gA{{\mathcal{A}}}

\def\gH{{\mathcal{H}}}

\def\gX{{\mathcal{X}}}
\def\gY{{\mathcal{Y}}}
\def\gZ{{\mathcal{Z}}}

\newcommand{\E}{\mathbb{E}}

\theoremstyle{plain}
\newtheorem{theorem}{Theorem}

\newtheorem{corollary}{Corollary}

\newtheorem{remark}{Remark}

\title{No-Free-Fairness: Fundamental Limits and Trade-offs in Learning Systems}

\author{%
  Khoat Than \\
  Hanoi University of Science and Technology\\
  \texttt{khoattq@soict.hust.edu.vn} \\
}

\begin{document}

\maketitle

\begin{abstract}
  In this paper, we establish a set of theoretical impossibility results, termed the \textit{No-Free-Fairness} theorems, that identify three fundamental sources of disparity in learning sysems. First, we show that when a task exhibits irreducible cost on a subgroup, any decision rule must trade off overall performance with disparity, yielding an inherent fairness--cost frontier. Second, we prove that even in ideal, noise-free settings where a perfectly fair and accurate solution exists, finite-sample learning alone induces nontrivial subgroup disparity, ruling out distribution-free fairness guarantees. More seriously, enforcing strict relative fairness creates a statistical bottleneck: achieving low cost may require exponentially many samples. Third, we show that limitations of the model class can independently induce disparity: if the model cannot represent accurate solutions for a subgroup, fairness remains unattainable regardless of data or training procedure.
  Overall, these results demonstrate that unfairness is not solely a consequence of biased data or suboptimal optimization, but arises from the intrinsic structure of decision problems, the constraints of finite data, and the expressivity of models. Our framework applies broadly beyond standard supervised learning, and suggests that achieving fairness requires explicit trade-offs and should be treated as a core design consideration.
\end{abstract}

\section{Introduction}

\begin{figure*}[t]
\centering
\includegraphics[width=\linewidth]{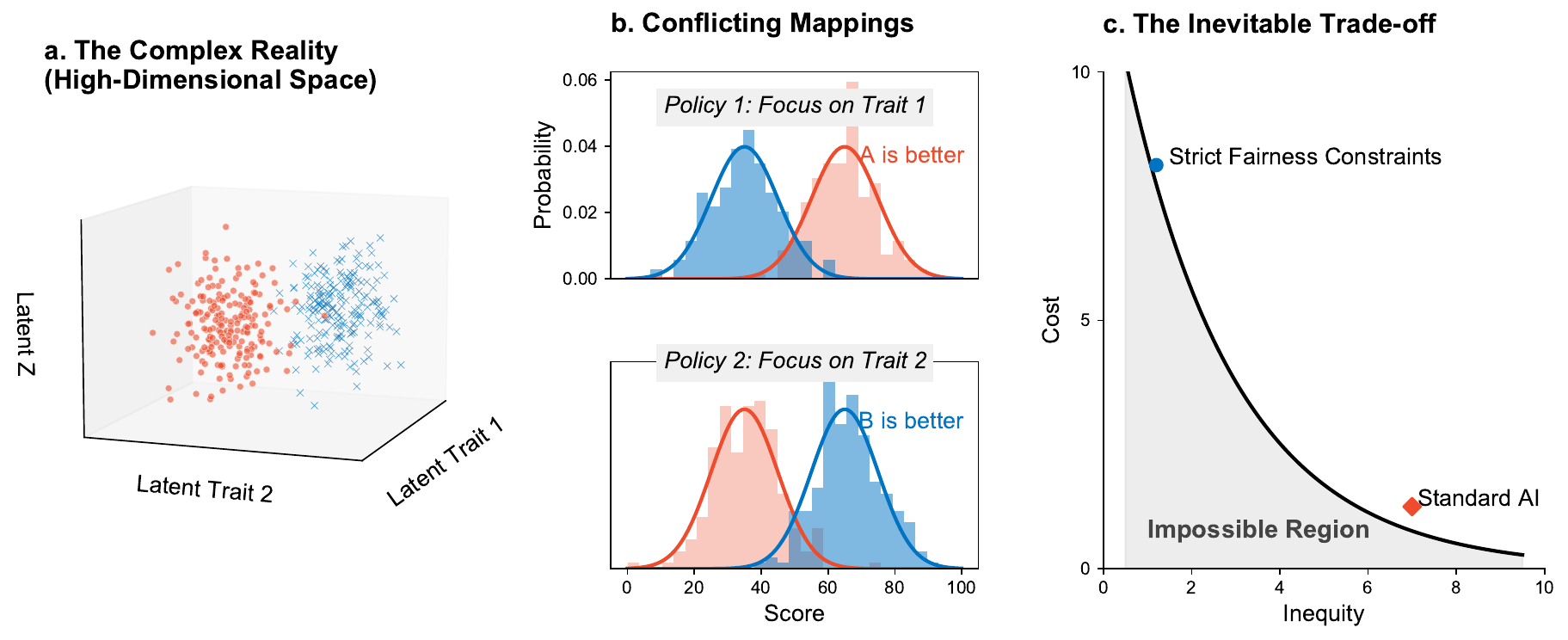}
\caption{Inherent Fairness--Cost Trade-off. (a) Individuals (or agents) are represented as points in a high-dimensional latent space capturing task-relevant attributes (traits). (b) Each task, together with its associated decision rule or policy, induces a mapping from this latent space to lower-dimensional outcomes. We illustrate that for some valid tasks, this mapping necessarily collapses information unevenly across subpopulations, leading to differential costs. (c) As a consequence, achievable systems lie on a fairness--cost Pareto frontier, where regions corresponding to simultaneously perfect fairness and minimal cost are theoretically unattainable.}
\label{fig:impossibility-vis}
\end{figure*}

As machine learning systems are increasingly deployed in high-stakes domains, from credit underwriting to generative AI, concerns regarding algorithmic fairness have moved to the forefront of the field \citep{kearns2018preventing,buolamwini2018gender,sheng2021societal,d2022underspecification}. While the community has proposed numerous fairness definitions and mitigation strategies \citep{dwork2012fairness,hardt2016equality,liu2018delayed,hashimoto2018fairness,zhang2024fairness,cheng2025biasfilter}, a growing literature establishes trade-offs between fairness criteria and predictive performance \citep{kleinberg2017inherent,chouldechova2017fair,pleiss2017fairness}. Most existing impossibility results, however, operate at the level of \emph{metric incompatibility}: they show that under differing base rates, certain fairness notions cannot be satisfied simultaneously. This leaves open three fundamental questions:

\begin{enumerate}
\item[Q1.] \emph{Does disparity emerge inevitably from the structural properties of statistical learning itself?} 
\item[Q2.] \emph{Can disparity always be resolved through improved data curation or algorithmic debiasing?}
\item[Q3.] \emph{Is there unavoidable trade-off  in  settings which admit perfectly fair and accurate solutions?}
\end{enumerate}

Prior work provides partial answers to Q1 and Q2. Using absolute disparity, \cite{cummings2019compatibility} show that one cannot simultaneously achieve differential privacy, exact fairness, and non-trivial accuracy, revealing a fundamental fairness--cost trade-off. \cite{pinzon2022impossibility} further establish the existence of tasks for which no predictor is both fair and accurate, even without privacy constraints. However, these results are largely confined to classification settings and rely on existential constructions. As such, \textbf{it remains unclear whether these limitations extend to more general learning problems, and whether they reflect isolated pathological cases or more pervasive phenomena.}

In this work, we study fairness through the \emph{risk ratio}, rather than the more commonly used absolute disparity \cite{cummings2019compatibility,pinzon2022impossibility}. While absolute disparity is analytically convenient and admits standard convergence guarantees \cite{woodworth2017learning,agarwal2018reductions,laakom25FairnessOverfitting}, it is inherently scale-dependent and can obscure disproportionate harm in low-risk regimes. In contrast, the risk ratio captures \emph{relative} differences in outcomes, is invariant to the overall scale of risk, aligns with ratio-based regulatory standards, and remains meaningful under class imbalance and rare-event settings (see Appendix~\ref{app:risk-ratio-justification}).

We establish a unified set of \textbf{No-Free-Fairness} results that identify fundamental limits to achieving fairness in learning systems. Moving beyond metric incompatibility and existential constructions, we show that disparity can arise from the basic mechanics of learning itself, and that such effects occur under broad conditions. Our contributions isolate three distinct sources of unavoidable disparity:
\begin{itemize}
    \item \textit{No-Free-Fairness from Data (Task-Inherent Limits).}    We show that when a problem is inherently uncertain for some subgroup, no decision rule can simultaneously achieve both high overall performance and equal outcomes across groups. In such settings, subgroup costs are intrinsically coupled: improving performance for one group necessarily affects another when task difficulty is heterogeneous. In other words, when some cases are fundamentally harder than others, perfect fairness is impossible. As a result, all decision rules are constrained by a fundamental \textit{ fairness--cost trade-off}, forming a Pareto frontier as depicted in Figure~\ref{fig:impossibility-vis}.

    \item \textit{No-Free-Fairness from Learning (Algorithmic and Statistical Limits).}    Even in ideal settings which admit perfectly fair and accurate decision rules, we show that learning from finite data introduces unavoidable disparities. In particular, small or rare groups are harder to learn accurately, which can lead to uneven performance. Moreover, enforcing strict fairness can require dramatically (exponentially many) more data, creating a fundamental \textit{tension between fairness and data efficiency}. Strictly enforcing near-perfect fairness can lead to non-vanishing overall risk, even with unlimited data.

    \item \textit{No-Free-Fairness from Model Classes (Structural Limits).}    We show that the choice of model itself can create fairness limitations. If a model lacks the flexibility to capture patterns specific to a subgroup, then unequal performance can persist even with unlimited, high-quality data. In this case, the limitation comes from how the model represents the problem, not from the data or the learning process.
\end{itemize}

The first and third results hold true for any non-negative cost function. Therefore, those results apply well to a wide range of contexts, e.g., regression, ranking, gererative tasks. Moreover, ``uncertainty" is ubiquitous in practice \citep{northcutt2021confident,jiang2018mentornet,song2022learning,koh2021wilds}. 
Taken together, our results show that fairness is not solely a consequence of biased data or imperfect algorithms. Rather, it is jointly determined by task structure, statistical estimation, and model expressivity. In particular, we provide strong evidence that (Q1) unfairness can emerge from intrinsic properties of learning,  that (Q2) universal elimination of disparity through improved data or algorithmic design alone is unattainable  in general, and that (Q3) there is a fundamental tension between fairness and data efficiency, even in ideal settings where perfectly fair and accurate solutions exist. Achieving fairness therefore requires explicit trade-offs and should be treated as a core design consideration.

\section{Impossibility Results} \label{sec-Impossibility}

This section presents our main impossibility results, touching on three axes: data distribution, algorithm, and hypothesis class. To facilitate detailed discussions, we need the following concepts.

\subsection{Preliminaries}

Let $\gX$ be a measurable space of individuals (or entities, states), and let $\gZ$ denote a space of outcomes capturing all stochastic elements of the environment. A (possibly randomized) decision rule  (or policy or hypothesis or model) is a measurable mapping $h: \gX \to \gY$, where $\gY$ can be a special case or subset of $\gZ$.

Let $(x,Z) \sim P$ be a joint distribution over $\gX \times \gZ$. We define a non-negative cost function $\ell(h(x),Z)$ to measure an undesirable property of the rule $h$ (e.g., loss, error, latency, or resource waste) for a specific instance $(x,Z)$, and the population \textit{risk} $a(h) = \E[\ell(h(x),Z)]$. For any measurable subset $\gX_b \subseteq \gX$ with $p_b = \mathbb{P}(x \in \gX_b) \in (0,1)$, define $a_b(h) = \E[\ell(h(x),Z)\mid x \in \gX_b]$ and  $a_{\bar b}(h) = \E[\ell(h(x),Z)\mid x \notin \gX_b]$.

A \emph{learner} (or decision-maker or learning algorithm) is a (possibly randomized) mapping $\gA$ that takes as input a dataset $S\sim P^n$ and outputs a decision rule $h_S = \gA(S)$. The \textit{learning task} is to produce a rule with small population risk $a(h_S)$.

\textbf{Instantiation of $\ell$ for common fairness notions.} Many standard fairness metrics can be expressed as group-conditional risks $a_b(h)$ via appropriate choices of $\ell$:
\begin{itemize}
\item \emph{Error parity.} For binary classification with $Z=y\in\{0,1\}$ and $h(x)\in\{0,1\}$, choosing $\ell(h(x),Z)=\mathbf{I}\{h(x)\neq Z\}$ yields $a_b(h)=\mathbb{P}(h(x)\neq y\mid x\in\gX_b)$, where $\mathbf{I}$ is the indicator function. So disparities correspond to differences (or ratios) in misclassification rates across groups.

\item \emph{Equalized odds.} Group-specific error rates are obtained by masking the loss: false negative rate (FNR) via $\ell(h(x),Z)=\mathbf{I}\{Z=1\}\mathbf{I}\{h(x)=0\}$, giving $a_b(h)=\mathbb{P}(h(x)=0\mid Z=1, x\in\gX_b)$; false positive rate (FPR) via $\ell(h(x),Z)=\mathbf{I}\{Z=0\}\mathbf{I}\{h(x)=1\}$, giving $a_b(h)=\mathbb{P}(h(x)=1\mid Z=0, x\in\gX_b)$.

\item \emph{Demographic parity.} This notion depends only on predictions and can be written with $\ell(h(x),Z)=\mathbf{I}\{h(x)=1\}$ (independent of $Z$), yielding $a_b(h)=\mathbb{P}(h(x)=1\mid x\in\gX_b)$.

\item \emph{Cost-based fairness.} In more general decision settings, let $Z$ encode outcomes and define $\ell$ as an application-specific cost. Then $a_b(h)$ measures expected cost within group $b$, and fairness corresponds to controlling disparities in these costs.
\end{itemize}

These instantiations show that a broad class of group fairness criteria can be unified as disparities in conditional risks under suitable choices of $\ell$. 

\textbf{Fairness in this work.} We study fairness through the \emph{risk ratio} $\varepsilon = a_{\bar{b}}(h)/a_b(h)$, rather than the more commonly used absolute disparity $\Delta(h) =|a_{\bar{b}}(h) - a_b(h)|$ \cite{woodworth2017learning,agarwal2018reductions,laakom25FairnessOverfitting,cummings2019compatibility,pinzon2022impossibility}. While the latter is analytically convenient and admits standard polynomial convergence guarantees, it is inherently scale-dependent and can obscure disproportionate harm in low-risk regimes. In particular, for fixed absolute gap $\Delta(h)$, the ratio satisfies $\varepsilon = 1 + \Delta(h)/a_b(h)$ (when $a_{\bar b}(h)\ge a_b(h)$), which becomes large as $a_b(h)\to 0$. In contrast, $\varepsilon$ directly captures multiplicative disparity, is invariant to global rescaling of risks, aligns with ratio-based regulatory criteria, and remains meaningful under class imbalance and rare-event settings. We provide a detailed justification for this choice in Appendix~\ref{app:risk-ratio-justification}.

\subsection{No-Free-Fairness Task}

We first investigate the structural properties of a learning task. In practice, a sampling procedure can induce a data distribution $P$ and hence determines the hardness of a learning task. The following theorem  reveals that the task itself can inherently induce equity, whose proof is in Appendix~\ref{app-Proofs-No-free-fairness-sampling}.

\begin{theorem}[No-Free-Fairness Task]\label{thm-No-free-fair-sampling}
Given a non-negative cost $\ell$, consider any distribution $P$ over a data space $\gX \times \gZ$ with $a(h^*)>0$, where $h^* \in \arg\min_h a(h)$ is the optimal one among all  measureable decision rules. Denote $\gX_b \subseteq \gX$ as the area with measure $p_b = \mathbb{P}(\gX_b)$ that incurs $a_b(h^*) >0$. Then any measurable rule $h$ satisfies 
\begin{equation}
\label{thm-No-free-fair-sampling-eq}
a(h) \ge \varepsilon a_b(h^*) \big[1 - p_b \mathbf{I}\{\varepsilon>1\}\big] 
\end{equation}
where $\varepsilon = \frac{a_{\bar{b}}(h)}{a_b(h)}$, meaning any decision rule is either bad or unfair.
\end{theorem}

This theorem applies to any decision-making problem in which a rule $h$ maps inputs $x$ to actions and incurs a non-negative cost $\ell(h(x),Z)$ under uncertainty $Z$. The quantity $a(h)$ represents the \emph{average cost} of the decision rule, while $a_b(h)$ and $a_{\bar b}(h)$ denote the average costs on two subpopulations. The ratio $\varepsilon = a_{\bar b}(h)/a_b(h)$ captures disparity in performance across groups.
The condition $a(h^*)>0$ simply states that the problem is inherently noisy or uncertain: even the optimal rule incurs nonzero cost on some region $\gX_b$. This is typical in real-world settings such as medical decision-making, resource allocation, or stochastic systems.

Despite having a surprizingly simple proof, Theorem~\ref{thm-No-free-fair-sampling} has several important implications:
\begin{itemize}
	\item \textit{Structural Limits of Universal Fairness.} Theorem~\ref{thm-No-free-fair-sampling} imposes a lower bound on the overall cost $a(h)$ as a function of disparity $\varepsilon$. In particular, enforcing near parity ($\varepsilon \approx 1$) exposes the irreducible cost on $\gX_b$, yielding $a(h)\gtrsim a_b(h^*)$. Writing $a(h)=p_b a_b(h)+(1-p_b)a_{\bar b}(h)$ shows that subpopulation costs are intrinsically coupled: improving outcomes for one group without affecting another is impossible when task difficulty is heterogeneous. Unlike prior impossibility results that rely on incompatibility between fairness criteria \citep{kleinberg2017inherent,chouldechova2017fair,pleiss2017fairness}, this limitation follows directly from the structure of expected cost and holds for any measurable rule and any non-negative $\ell$.
	
	\item \textit{Quantified Fairness--Cost Trade-off.} The bound (\ref{thm-No-free-fair-sampling-eq}) defines a feasible region of achievable pairs $(a(h),\varepsilon)$. When $\varepsilon \le 1$, achieving near-uniform performance forces $a(h)=\Omega(a_b(h^*))$, so the intrinsic difficulty on $\gX_b$ sets a fundamental cost of fairness. When $\varepsilon>1$, the factor $(1-p_b)$ attenuates the bound, reflecting that smaller groups contribute less to $a(h)$. This explains why fairness costs may appear small in aggregate metrics under class or group imbalance, even though disparities persist structurally. These observations are consistent with prior theoretical and empirical findings on fairness--cost trade-offs \citep{dwork2012fairness,hardt2016equality}.
	
	\item \textit{Task Variability and Lack of Uniform Guarantees.}	The trade-off depends on $a_b(h^*)$, which is determined by the underlying outcome distribution. Fixing the marginal $P_X$ of $x$ while varying the conditional structure of the environment changes the intrinsic difficulty across subpopulations and thus the achievable $(a(h),\varepsilon)$ region. Consequently, a rule that attains a favorable trade-off for one task may violate the bound under another. No method can guarantee uniformly low cost and low disparity across all admissible environments, highlighting an inherent sensitivity to distributional changes, consistent with observations in distribution shift settings \citep{koh2021wilds}.
\end{itemize}

Theorem~\ref{thm-No-free-fair-sampling} formalizes a general principle: \textit{whenever a decision problem exhibits irreducible uncertainty, one must trade off overall performance against equality across groups.} This limitation is structural, independent of the specific model or algorithm, and persists across application domains.

\begin{remark}[On the practicality of the condition $a(h^*)>0$] \label{remark-practicality-tasks}
The condition $a(h^*)>0$ corresponds to the standard notion of a \emph{non-realizable} setting in learning theory: there is no decision rule that achieves zero population cost under $P$. Far from being restrictive, this is the typical regime in practice. In real-world problems, noise, ambiguity, unobserved confounders, and intrinsic stochasticity in the outcome mechanism (e.g., $P_{Y|X}$ in supervised learning) imply that even the optimal rule incurs nonzero cost on some subset $\gX_b$, i.e., $a_b(h^*)>0$ \citep{zhang2021understandingDL,arora2019fine}. Moreover, many application domains are inherently noisy or ill-defined---such as medical diagnosis, human annotation, or natural language understanding---where multiple plausible outcomes or measurement noise make perfect decisions unattainable \citep{northcutt2021confident,jiang2018mentornet,song2022learning}. Distribution shift \citep{koh2021wilds} further reinforces this point: even if a problem appears realizable under one distribution, it may be non-realizable when the environment changes.

Therefore, $a(h^*)>0$ should be viewed as a mild and ubiquitous property of realistic decision problems. Importantly, it is precisely this irreducible cost that makes fairness nontrivial: if zero cost were achievable everywhere, disparities could be eliminated without trade-offs. Theorem~\ref{thm-No-free-fair-sampling} thus characterizes the practically relevant regime in which fairness--cost trade-offs are unavoidable.
\end{remark}

\subsection{No-Free-Fairness Algorithm}

The previous result shows the scenarios, where the learning tasks may not be learnable. However, for some tractable tasks (where $P$ is \textit{realizable}, i.e., $a(h^*)=0$), one may hope to find some fair solutions, by careful modeling and data curation. Unfortunately, it may not always be the case as revealed by the following result for binary classification, whose proof appears in Appendix~\ref{app-Proofs-No-free-fairness-algorithm}.

\begin{theorem}[No-Free-Fairness  Algorithm]\label{thm-No-free-fair-algorithm}
Consider any measurable input space $\gX$ which admits a non-atomic distribution, 0-1 loss $\ell$, any real number $p_b \in (0,1)$, any $c \in (0,1]$. Denote $\mathcal{A}$ as a (possibly randomized) learner  that maps each dataset $S$ to a decision rule $\gA(S)$, and $\varepsilon_S = \frac{a_{\bar b}(\gA(S))}{\max\{a_b(\gA(S)), c\}}$ as the cost ratio.
Let $\mathcal{P}_{\text{real}}(p_b)$ be the class of all realizable probability distributions over $\mathcal{X} \times \{0,1\}$ that contain a subset $\mathcal{X}_b$ with marginal probability $p_b$. Then 
\begin{eqnarray}
\inf_{\gA} \sup_{P \in \mathcal{P}_{\text{real}}(p_b)} \mathbb{E}_{S \sim P^n} [\varepsilon_S] &\ge& 1/4, \\
\inf_{\gA} \sup_{P \in \mathcal{P}_{\text{real}}(p_b)} \frac{\mathbb{E}_{S \sim P^n} [a(\gA(S))]}{\mathbb{E}_{S \sim P^n} [\varepsilon_S]} &\ge& c(1 - p_b).
\end{eqnarray}
\end{theorem}

This theorem establishes a fundamental \emph{minimax} lower bound on subgroup disparity, formalizing an unavoidable trade-off between overall cost and fairness. In simple terms, it shows that with finite data, no learning procedure can consistently avoid unequal performance across groups—even in ideal, noise-free settings where a perfectly fair and accurate rule exists. Some inportant implications thus follow.

\begin{itemize}
\item \textit{Minimax Limits and Algorithm-Induced Unfairness.} Theorem~\ref{thm-No-free-fair-algorithm} provides a distribution-free, minimax lower bound on disparity in the \emph{realizable} regime. In particular, the relation $\frac{\E_{S}[a(\gA(S))]}{\E_{S}[\varepsilon_S]} \ge c(1-p_b)$ links disparity to overall cost. Thus, even when there exists a rule $h^*$ with zero cost everywhere ($a(h^*)=0$), \emph{no learner can reliably produce low-disparity solutions from finite samples}. This is a minimax statement: the limitation arises from worst-case (but still realizable) data-generating processes and the randomness of sampling, not from noise or model mismatch. In contrast, Theorem~\ref{thm-No-free-fair-sampling} gives a \emph{pointwise} lower bound in the non-realizable case, where irreducible cost $a_b(h^*)>0$ directly forces a trade-off. Here, even when $a_b(h^*)=0$, sampling alone can induce constant disparity in expectation.

\item \textit{Limitations of the ``Clean Data'' Narrative.} A common belief is that unfairness mainly stems from biased or noisy data \citep{FoundationModels21,Trewin2019fairness,buyl2024inherentLimits,calmon2017optimizedFair,Caton2024Fairness}. Theorem~\ref{thm-No-free-fair-algorithm} shows that this explanation is incomplete. Even with perfectly clean and consistent data (realizability), one can still have $\E_{S}[\varepsilon_S]\ge 1/4$. Hence, \emph{good data alone does not guarantee fairness}. The underlying reason is statistical: subpopulations with small probability mass $p_b$ are underrepresented in finite samples, making their behavior harder to learn. As a result, a learner may achieve low overall cost while performing poorly on these regions. This complements Theorem~\ref{thm-No-free-fair-sampling}, where irreducible cost drives the trade-off; here, \emph{sampling variability alone} creates disparity. This phenomenon is closely related to underspecification \citep{d2022underspecification}, where multiple solutions fit the data equally well but differ significantly on subgroups.

\item \textit{Fairness as a No-Free-Lunch Phenomenon.} This result can be viewed as a fairness analogue of the No Free Lunch theorem \citep{wolpert1997no}. While classical results show that no algorithm can uniformly minimize overall cost across all distributions, here we show a stronger, localized limitation: without additional assumptions on $P$ or on the structure of $\gX_b$, no learner can guarantee uniformly small disparity.  $\varepsilon_S$ can reach $1/c$, meaning that the disparity can be high to ensure small overall cost. Achieving $\E_{S}[\varepsilon_S] \approx 1$ therefore requires \emph{inductive bias}, such as structural assumptions, regularity, or explicit fairness constraints. Together with Theorem~\ref{thm-No-free-fair-sampling}, this yields a unified picture: (i) \emph{algorithmic limits} prevent distribution-free fairness guarantees even when perfect solutions exist, and (ii) \emph{structural limits} (Theorem~\ref{thm-No-free-fair-sampling}) impose trade-offs when irreducible cost is present. Fairness cannot be made universal: it depends on the problem structure and the assumptions imposed.

\end{itemize}

\begin{remark}[Practical relevance of threshold $c$]
We use $c$ to define $\varepsilon_S$ not only to avoid degeneracy when \(a_b(h_S)\) of a model $h_S$ is small, but also to encode a \emph{minimum level of statistical relevance}. When \(a_b(h_S) < c\), we have $\varepsilon_S = a_{\bar b}(h_S)/c$ which is strictly less than the true disparity $a_{\bar b}(h_S) /a_{b}(h_S)$; it is possible that \(a_b(h_S)\) is negligible while \(a_{\bar b}(h_S)\) remains large, in which case the disparity is substantial and $\varepsilon_S$ remains informative. 

The practical relevance issue arises specifically in the regime where \emph{both} \(a_b(h_S)\) and \(a_{\bar b}(h_S)\) are small (e.g., below \(c\)). In this case, the ratio \(a_{\bar b}(h_S)/a_b(h_S)\) can be dominated by negligible fluctuations, making large relative disparities practically uninformative. When \(a_b(h_S) < c\), the metric effectively transitions to \(\varepsilon_S \approx a_{\bar b}(h_S)/c\), emphasizing \emph{absolute cost control} rather than relative disparity. Thus, \(c\) serves as a \emph{relevance threshold}: above \(c\), \(\varepsilon_S\) reflects relative fairness; below \(c\), it  deprioritizes ratio-based comparisons when both costs are already negligible.

This relevance issue is not unique to ratio-based metrics. Difference-based measures such as \(\lvert a_{\bar b}(h_S) - a_b(h_S)\rvert\)  are inherently \emph{scale-insensitive}: a gap of the same magnitude is treated equally regardless of whether costs are large or near zero. Consequently, such metrics may overemphasize negligible absolute gaps in low-cost regimes, while ratio-based metrics without a floor may overstate disparities. From this perspective, the floor \(c\) provides a principled interpolation by enforcing a minimum scale below which relative comparisons are no longer considered meaningful.
\end{remark}


\subsubsection{Fairness vs. Convergence Rates: A Statistical Tension}

Theorem~\ref{thm-No-free-fair-algorithm} introduces a regularization floor $c \in (0,1]$ in the disparity metric
$\varepsilon_S = a_{\bar b}(h_S) / \max\{a_b(h_S),\,c\}$, where $h_S = \gA(S)$. If $c$ is fixed, then there exist realizable distributions for which $\mathbb{E}[\varepsilon_S]=\Omega(1)$ while simultaneously $\mathbb{E}[a(h_S)] = \Omega(1)$. This shows that \emph{distribution-free statistical consistency is incompatible with enforcing a relative fairness}. On the other hand, allowing $c=c(n)$ to depend on $n$ reveals a fundamental statistical tension, as shown below.

\begin{corollary}\label{cor:rate-lower-bound}
$\inf\limits_{\gA} \sup\limits_{P \in \mathcal{P}_{\mathrm{real}}(p_b)} \E_{S\sim P^n} \left[a(\gA(S))\right] \ge \frac{1-p_b}{4} c(n)$, for every $n$ and $c(n) \in (0,1]$.
\end{corollary}

This result implies that any distribution-free guarantee on the expected risk must satisfy $\sup_{P}\mathbb{E}[a(h_S)] = \Omega(c(n))$. Consequently, enforcing a relatively strict and slowly decaying floor (e.g., $c(n)=1/\log n$) yields correspondingly slow convergence rates. In particular, achieving statistical consistency (i.e., $\mathbb{E}[a(h_S)] \to 0$) necessarily requires $c(n)\to 0$. 

\textbf{A Statistical Tension.} The above analysis shows that $c(n)$ governs a fundamental trade-off. Bounding the maximum possible disparity requires a non-vanishing $c(n)$, but this induces a non-vanishing worst-case cost floor that strictly precludes statistical consistency. Conversely, achieving vanishing cost necessitates $c(n)\to 0$, under which $\varepsilon_S$ increasingly reduces to raw disparity $\frac{a_{\bar b}(h_S)}{a_b(h_S)}$. Because the worst-case upper bound of this ratio scales as $1/c(n)$, allowing $c(n)\to 0$ leaves the disparity loosely constrained and mathematically permits severe unfairness. This reveals an inherent tension between distribution-free fairness guarantees and classical consistency.

\textbf{A Bottleneck of Sample Complexity.} Corollary~\ref{cor:rate-lower-bound} shows that $\sup_{P}\E_S[a(h_S)] = \Omega(c(n))$, revealing a fundamental statistical bottleneck specific to \emph{relative} fairness. Enforcing a nontrivial lower bound $c(n)$ on subgroup costs (to control ratios) directly limits how fast the overall cost can decay. Concretely, the choice of $c(n)$ determines preference on sample complexity or  equity:
\begin{itemize}
    \item \textit{Prioritizing fairness leads to exponential sample complexity.}
    If $c(n)$ decays slowly (e.g., $c(n)=1/\log n$), then subgroup costs may not easily become too small, which tightly controls the disparity $\varepsilon_S$. However, achieving $\E_S[a(h_S)] \le \nu$ now requires $c(n)\lesssim \nu$, implying $n \ge e^{\Omega(1/\nu)}$. In this regime, the interaction between rare subpopulations and strict ratio constraints converts the usual sublinear sample complexity into an exponential one.
    \item \textit{Prioritizing convergence leads to vacuous fairness.}
    If $c(n)$ decays rapidly (e.g., $c(n)=1/n$ or $1/n^2$), standard polynomial convergence rates are recovered. However, the worst-case disparity scales as $\varepsilon_S = \Omega(1/c(n))$, so $c(n)=\mathcal{O}(1/n^2)$ allows $\varepsilon_S=\Omega(n^2)$. In this regime, relative fairness guarantees can  become vacuous, as arbitrarily large disparities are permitted.
\end{itemize}

\textit{The Relevance Threshold.} This dichotomy reveals a pathology of ratio-based fairness metrics in large-sample regimes. As $n\to\infty$, absolute costs $a_b(h_S)$ and $a_{\bar b}(h_S)$ both tend to $0$. If they decay at slightly different rates (e.g., $1/n$ vs.\ $1/n^2$), their difference vanishes while their ratio $\varepsilon_S$ diverges. The floor $c(n)$ acts as a \emph{relevance threshold}: when $a_b(h_S)\gtrsim c(n)$, the ratio $\varepsilon_S$ enforces meaningful relative comparisons; when $a_b(h_S)\ll c(n)$, the ratio is effectively disregarded. This formalizes the intuition that once costs are negligible across all groups, amplifying vanishing differences is not meaningful. Consequently, the choice of $c(n)$ encodes a principled transition between prioritizing strict equity and prioritizing overall utility.

\subsection{No-Free-Fairness with Limited Model Classes}

We have seen that disparity can arise from the task itself or from the learning procedure. We now examine a third factor: the role of model capacity and inductive bias.

\begin{theorem}[Limited Models]\label{thm-No-free-fair-Architecture}
Consider any distribution $P$ over $\gX \times \gY$, a non-negative cost $\ell$, and a  class $\gH$ of decision rules/hypotheses. If there exists a subset $\gX_b \subseteq \gX$ with measure $p_b = \mathbb{P}(\gX_b)$ such that $0 < a_b^* := \min_{g \in \gH} a_b(g)$, then any model $h \in \gH$ satisfies 
\[
a(h) \ge \varepsilon\, a_b^* \big[1 - p_b \mathbf{I}\{\varepsilon>1\}\big], 
\quad \text{where } \varepsilon = \frac{a_{\bar{b}}(h)}{a_b(h)}.
\]
\end{theorem}

This theorem highlights the role of \emph{model capacity}. A sufficiently expressive class $\gH$ may achieve $a_b^* = 0$, while a limited or misspecified class can have $a_b^*>0$, meaning that no model in $\gH$ performs well on $\gX_b$. This leads to the following implications:

\begin{itemize}
\item \textit{Model Capacity as a Source of Fairness Limits.}
Theorem~\ref{thm-No-free-fair-Architecture} identifies model capacity as a third, distinct source of disparity. While Theorem~\ref{thm-No-free-fair-sampling} attributes unfairness to inherent task difficulty and Theorem~\ref{thm-No-free-fair-algorithm} to finite-sample effects, this result shows that \emph{the choice of model class alone} can enforce a trade-off. The condition $a_b^*>0$ means that the class $\gH$ cannot represent low-cost solutions on $\gX_b$. As a result, improving overall performance $a(h)$ tends to reduce $a_{\bar b}(h)$ while leaving $a_b(h)\ge a_b^*$, driving $\varepsilon \to 0$. Conversely, enforcing parity ($\varepsilon \approx 1$) requires increasing error outside $\gX_b$, leading to $a(h)\gtrsim a_b^*$. Thus, any nonzero subgroup floor $a_b^*$ induces a fairness--accuracy trade-off.

\item \textit{Representation Gaps and Feature Limitations.}
The quantity $a_b^*>0$ captures a \emph{representation gap}: the model class $\gH$ lacks the expressiveness needed to perform well on $\gX_b$. In practice, this can arise from limited features, restrictive architectures, or inductive biases that favor majority patterns. Unlike classical irreducible cost due to noise or ambiguity \citep{northcutt2021confident,song2022learning,koh2021wilds}, this limitation is \emph{model-dependent}: a better representation or richer class could reduce $a_b^*$. Hence $a_b^*$ can be interpreted as \emph{capacity-induced irreducibility}.

\item \textit{Cost Reallocation under Capacity Constraints.}
When $a_b^*>0$, minimizing the overall cost $a(h)=p_b a_b(h)+(1-p_b)a_{\bar b}(h)$ necessarily redistributes cost across groups. Optimizing within $\gH$ tends to reduce $a_{\bar b}(h)$ while leaving $a_b(h)$ bounded below by $a_b^*$, leading to $\varepsilon \to 0$. Enforcing $\varepsilon \approx 1$ instead requires increasing $a_{\bar b}(h)$, yielding $a(h)\gtrsim (1-p_b)a_b^*$. Thus, the subgroup with positive $a_b^*$ acts as a \emph{bottleneck} that determines the achievable trade-off. Unlike Theorem~\ref{thm-No-free-fair-sampling}, where the bottleneck is due to the data distribution, here it is determined by the model design.
\end{itemize}

Prior work often attributes unfairness primarily to biased or unrepresentative data \citep{bender2021dangers,FoundationModels21}. Theorem~\ref{thm-No-free-fair-Architecture} shows that this view is incomplete: even with abundant, balanced, and noise-free data, disparity can persist if $\gH$ cannot represent accurate decisions on $\gX_b$. The quantity $a_b^*>0$ depends only on the model class and not on sampling. Thus, data-centric interventions alone cannot eliminate unfairness when the model is misspecified. 

Together with Theorem~\ref{thm-No-free-fair-algorithm} and Theorem~\ref{thm-No-free-fair-sampling}, this result provides a unified picture: fairness limitations arise from three distinct sources—\emph{data}, \emph{algorithm}, and \emph{architecture}. Each imposes its own constraint, and none can be fully resolved by improving the others alone.

\textbf{Connection to Modern Models (e.g., LLMs).} Recent empirical studies have documented various forms of bias and inequity in LLMs, including demographic bias \citep{bolukbasi2016man,sheng2021societal,jeung2025large, reddy2025metamorphic}, representational harms \citep{abid2021persistent,FoundationModels21}, and disparities in toxicity or refusal behavior across social groups \citep{gehman2020realtoxicityprompts} or social biases in health care \citep{liu2025potentialHealth}. \citet{gallegos2024bias} highlight the breadth of social bias phenomena and mitigation approaches in LLMs. Both generative behaviors and identity biases have been observed even after fine-tuning and alignment, indicating challenges in eliminating harmful patterns \citep{hu2025biases}. Theorem~\ref{thm-No-free-fair-Architecture} suggests a structural explanation: if the deployed architecture induces $a_b^*>0$ for certain subpopulations, then disparity is unavoidable regardless of data scale or alignment procedures. This motivates architectural interventions for debiasing, such as richer representations, subgroup-aware modules, or adaptive capacity allocation.

\section{Comparison with Prior Impossibility Results.}

The fundamental tension between predictive accuracy and algorithmic fairness has been extensively studied, yet existing theories isolate different (and often orthogonal) sources of this tension. Our framework unifies and sharpens these perspectives by identifying three distinct mechanisms: \emph{task-inherent} (Theorem~\ref{thm-No-free-fair-sampling}), \emph{algorithm-inherent} (Theorem~\ref{thm-No-free-fair-algorithm}), and \emph{architecture-inherent} (Theorem~\ref{thm-No-free-fair-Architecture}). Below, we contrast these mechanisms with prior work along several key dimensions.

\textbf{Task-Inherent Bias, Algorithmic Cost, and Statistical Bottlenecks.}
\citet{cummings2019compatibility} show that a good classifier cannot simultaneously achieve both differential privacy and exact fairness (equal opportunity), revealing a fundamental fairness--cost trade-off. \citet{pinzon2022impossibility} further establish the existence of tasks for which no predictor is both accurate and exactly fair. In these works, the tension arises because fairness constraints force the learner to deviate from minimizing $a(h)$, leading to gaps of the form $a(h_{\text{fair}}) - a(h^*)$ that depend on the chosen constraint.

In contrast, our results show that a nontrivial trade-off exists \emph{even before imposing any constraint}. In Theorems~\ref{thm-No-free-fair-sampling} and \ref{thm-No-free-fair-Architecture}, the lower bound $a(h) \ge \varepsilon\, a_b^\dagger \big[1 - p_b \mathbf{I}(\varepsilon>1)\big]$, where $a_b^\dagger \in \{a_b(h^*),\, a_b^*\}$, is intrinsic to the problem. Here $a_b(h^*)$ captures irreducible (task-inherent) cost, while $a_b^*=\min_{g\in\gH} a_b(g)$ captures capacity-induced cost. Thus, the fairness--cost frontier is fundamentally constrained by \emph{statistical or representational heterogeneity}, rather than by externally imposed constraints. More importantly, while \citet{pinzon2022impossibility} establish task limitations through carefully constructed worst-case tasks, our Theorem~\ref{thm-No-free-fair-sampling} shows that such limitations arise naturally in realistic settings, as discussed in Remark~\ref{remark-practicality-tasks}.

Beyond these structural limits, our analysis also reveals a distinct \emph{statistical bottleneck} for relative fairness. Corollary~\ref{cor:rate-lower-bound} shows that enforcing a nontrivial fairness floor $c(n)$ directly constrains the achievable convergence rate. In particular, maintaining strict relative fairness (slowly decaying $c(n)$) can force \emph{exponential sample complexity}, while recovering classical rates (rapidly decaying $c(n)$) renders fairness guarantees vacuous. This phenomenon has no analogue in prior analyses.

\textbf{Metric Incompatibility vs. Intrinsic Disparity.} Foundational impossibility results \citep{kleinberg2017inherent,chouldechova2017fair,pleiss2017fairness} show that multiple fairness criteria (e.g., calibration and equalized odds) cannot be simultaneously satisfied when base rates differ: $P(Y=1\mid \gX_b) \neq P(Y=1\mid \gX_{\bar b})$. Many fairness metrics often rely on concepts of independence, separation, and sufficiency, which cannot hold \textit{simultaneously} in non-trivial settings \citep{berk2021fairness}. These are \emph{multi-metric incompatibility} results: impossibility arises from conflicting definitions.

Our framework instead considers a \emph{single disparity measure}, $\varepsilon = a_{\bar b}(h)/a_b(h)$, and shows that even this minimal notion cannot be optimized jointly with accuracy. Hence, disparity persists \emph{without invoking incompatible constraints}. Moreover, unlike \citet{dwork2012fairness}, which requires a task-specific similarity metric to define individual fairness, our results operate purely at the level of expected risks, showing that even coarse group-level parity is fundamentally constrained. This demonstrates that unfairness is not merely a definitional artifact, but a consequence of heterogeneous error landscapes.

\textbf{Beyond Classification: Loss-Agnostic Fairness Limits.} A further distinction is that much of the existing theory is inherently tied to classification, where fairness notions are defined via confusion matrix quantities such as FPR and FNR  \citep{hardt2016equality,kleinberg2017inherent,chouldechova2017fair,pleiss2017fairness}. These formulations do not easily extend to regression, ranking, or generative modeling.

In contrast, our results in Theorems \ref{thm-No-free-fair-sampling} and \ref{thm-No-free-fair-Architecture} are expressed in terms of a general non-negative cost $\ell$ and subgroup risks $a_b(h)$, and therefore apply uniformly across learning settings. This shows that the fairness--accuracy tension is not a peculiarity of classification metrics, but a general property of risk minimization under heterogeneous populations. Consequently, our framework extends the scope of fairness impossibility results to a broader class of modern learning problems, including regression and generative modeling, where prior theory provides limited guidance.

\textbf{Distributional Assumptions vs. Minimax Finite-Sample Limits.} Many prior results rely on specific distributional conditions (e.g., differing base rates or noise levels). In contrast, Theorem~\ref{thm-No-free-fair-algorithm} establishes a \emph{distribution-free minimax bound} for $\mathbb{E}_S[\varepsilon_S]$. This holds even in realizable settings where $a_b(h^*)=a_{\bar b}(h^*)=0$. Thus, unfairness can arise purely from finite-sample variability, without noise or structural bias. Together with Theorem~\ref{thm-No-free-fair-sampling}, this yields a unified picture: \textit{disparity is unavoidable both in finite-sample regimes and in asymptotic regimes with irreducible error.}

\textbf{Rates and Sample Complexity under Fairness Constraints.} Beyond existence results \cite{cummings2019compatibility,pinzon2022impossibility}, our framework yields quantitative implications for statistical rates. Corollary~\ref{cor:rate-lower-bound} shows that any distribution-free guarantee must satisfy $\sup_{P}\mathbb{E}_{S\sim P^n}[a(h_S)] \;=\; \Omega(c(n))$, thereby constraining how the expected risk can decay as a function of the sample size $n$. In particular, achieving statistical consistency (i.e., $\mathbb{E}[a(h_S)] \to 0$ as $n\to\infty$) requires $c(n)\to 0$, which effectively relaxes the fairness constraint and permits larger disparity. Conversely, maintaining a non-vanishing, scale-stable notion of fairness (constant $c$) induces a non-vanishing worst-case error floor.

This result seems orthogonal to some prior studies \cite{agarwal2018reductions,donini2018empirical,woodworth2017learning,cummings2019compatibility} which show polynomial convergence rates for the expected risk. The main reason for this difference comes from the use of absolute disparity such as $| a_b(h) - a(h)|$ in \cite{agarwal2018reductions} or $| a_b(h) - a_{\bar b}(h)|$ in \cite{donini2018empirical,woodworth2017learning,cummings2019compatibility}. As discussed before, while absolute disparity is  convenient for analysis, it is inherently scale-dependent and can obscure disproportionate harm in low-risk regimes. Therefore, despite providing various insights, prior guarantees with polynomial rates may not reflect the true cost when imposing fairness constraints in practice.

In contrast, the use of risk ratio in our framework can better align with ratio-based regulatory standards. Our results show that  fairness constraints fundamentally limit the achievable \emph{statistical rates} and sample complexity in a distribution-free setting. The limitation arises even with perfect empirical risk minimization, and is  driven by statistical variability and subgroup structure.


\textbf{Architecture and Representation vs. Data-Centric Views.} Data-centric perspectives \citep{bender2021dangers,FoundationModels21} attribute bias primarily to skewed or noisy datasets. Theorem~\ref{thm-No-free-fair-Architecture} shows that even with ideal data, the hypothesis class can induce disparity. This connects to approximation theory: restricting $\gH$ induces non-uniform approximation error across the population, which translates into fairness gaps. Hence, disparity can persist even under perfect data curation, unless the model class is sufficiently expressive. This complements empirical findings on representation bias and subgroup performance gaps, and suggests that \emph{architectural capacity is a crucial determinant of fairness}.

\textit{Implications for Mitigation Strategies.} Existing mitigation approaches, such as post-hoc thresholding \citep{hardt2016equality} or data repair \citep{feldman2015certifying}, operate within a fixed model class $\gH$ and dataset. Our results show that such methods may not overcome structural limits imposed by $a_b^\dagger>0$ or finite-sample bounds. Therefore, meaningful mitigation may require \emph{changing the problem itself}: enriching features, increasing subgroup coverage, or expanding $\gH$ to reduce $a_b^*$. This reframes fairness as a systems-level design problem rather than a purely algorithmic one.

\section{Conclusion} \label{sec-Conclusion}
We establish a unified set of \emph{No-Free-Fairness} results that characterize fundamental limits of achieving fairness in learning systems along three axes. These results move beyond metric incompatibility and existential constructions, demonstrating that disparity can arise from general and non-pathological properties of tasks, learners, and models. In this sense, disparity is not merely a consequence of flawed data or misaligned optimization, but a structural feature of statistical learning that must be explicitly managed.

Our analysis focuses on relative fairness measured by the risk ratio and primarily studies population risks and worst-case guarantees, which may not capture all practically relevant notions of fairness or distributional structure encountered in real systems. While our results extend beyond standard classification settings, they do not fully address dynamic, strategic, or feedback-driven environments where decisions influence future data. Finally, although we argue that the identified phenomena arise under broad conditions, quantifying their exact severity in specific real-world applications remains an open empirical question. These limitations point to important directions for future work. 

\section{Broader Impact} \label{sec-Broader-Impact}

This work studies fundamental limits of fairness in  learning systems. Our results suggest that unfairness cannot always be eliminated through better optimization, larger datasets, or post-hoc correction alone, because disparities may arise from intrinsic uncertainty, finite-sample effects, or limitations of the model class itself. By formalizing these constraints, our work may help practitioners and policymakers develop more realistic expectations regarding the capabilities and limitations of AI systems deployed in high-stakes settings.

A potential positive impact of this work is to encourage more transparent evaluation of fairness trade-offs and more principled design of learning systems, particularly for underrepresented or difficult subpopulations. Our framework may also motivate future research on adaptive data collection, robust model design, and application-specific notions of equity.

At the same time, these impossibility results could be misinterpreted as suggesting that fairness efforts are futile. This is not the intent of our work. Rather, our results emphasize that fairness requires explicit assumptions, design choices, and trade-offs, and cannot generally be guaranteed in a distribution-free manner. Understanding these limitations is important for developing more reliable and accountable AI systems.

\bibliography{fair,ann,complexity}
\bibliographystyle{abbrvnat}


\newpage
\appendix
\onecolumn

\section{Proof for No-free-fairness Task} \label{app-Proofs-No-free-fairness-sampling}

\begin{proof}[Proof of Theorem \ref{thm-No-free-fair-sampling}]

First, we observe that $a_b(h) \ge a_b(h^*) >0$, due to the optimality of $h^*$. We next analyze two cases.

\textbf{Case 1:} $\varepsilon \le 1$. Observe that
\begin{eqnarray}
a(h) &=& p_b a_b(h) + (1-p_b) a_{\bar{b}}(h) \\
&=&  p_b a_b(h) + (1-p_b) \varepsilon a_b(h) \\
&=& (1-\varepsilon) p_b a_b(h) + \varepsilon a_b(h) \\
&\ge& \varepsilon a_b(h) \\
&\ge& \varepsilon a_b(h^*).
\end{eqnarray}

\textbf{Case 2:} $\varepsilon > 1$. Observe that
\begin{eqnarray}
a(h) &=& p_b a_b(h) + (1-p_b) a_{\bar{b}}(h) \\
&=&  p_b a_b(h) + (1-p_b) \varepsilon a_b(h) \\
&\ge& \varepsilon(1-p_b) a_b(h) \\
&\ge& \varepsilon(1-p_b) a_b(h^*).
\end{eqnarray}

Combining those two cases lead to the following:
\begin{eqnarray}
a(h) &\ge& 
\begin{cases}
	\varepsilon(1-p_b) a_b(h^*) & \text{ if } \varepsilon >1\\
	\varepsilon a_b(h^*) & \text{ otherwise } 
\end{cases} \\
&=& \varepsilon a_b(h^*)
\begin{cases}
	1-p_b & \text{ if } \varepsilon >1\\
	1 & \text{ otherwise } 
\end{cases} \\
&=& \varepsilon a_b(h^*) \big[1 - p_b \mathbf{I}\{\varepsilon>1\}\big].
\end{eqnarray}
\end{proof}

\section{Proof for No-free-fairness algorithm} \label{app-Proofs-No-free-fairness-algorithm}

Theorem \ref{thm-No-free-fair-algorithm} is a corollary of the following result.

\begin{theorem}\label{thm:realizable_fixed}
Consider any measurable input space $\gX$ which admits a non-atomic distribution,  any (possibly randomized) learning algorithm $\mathcal{A}$ that maps each dataset $S$ to $h_S:\gX \to \{0,1\}$,  0-1 loss, any real number $p_b \in (0,1)$, any $c \in (0,1]$ that may depend on $n$ and $p_b$.

Then there exist a realizable distribution $P$ over $\gX \times \{0,1\}$ and a measurable subset $\gX_b \subseteq \gX$ with $\mathbb{P}(\gX_b) = p_b$ such that:
\begin{eqnarray}
\E_{S}[\varepsilon_S] &\ge& 1/4, \\
\label{eq:main_lb_fixed-02}
\E_{S}[a(h_S)] &\ge& c(1-p_b) \E_{S}[\varepsilon_S], 
\end{eqnarray}
where $\varepsilon_S = {a_{\bar b}(h_S)}/{\max\{a_b(h_S), c\}}$.
\end{theorem}

\begin{proof}
Choose any integer $m \ge 2n$ where $n$ is the number of samples in $S$. The proof contains three main steps.

\textbf{Step 1: Partition of $\gX$ and define $P$.}
By the assumption of $\gX$, there exist a partition $\gX = \gX_b \cup \gX_1 \cup \cdots \cup \gX_m$ of $\gX$ into  disjoint measurable sets, and  a marginal distribution $P_X$ over $\gX$ so that 
\begin{equation}
\label{eq:px_def_fixed}
P_X(\gX_b) = p_b, \qquad
P_X(\gX_i) = \frac{1-p_b}{m}, \quad i \in [m].
\end{equation}

Draw $\theta = (\theta_1,\dots,\theta_m) \sim \mathrm{Unif}(\{0,1\}^m)$. Define $P_{Y|X}$ by choosing
\begin{equation}
\label{eq:pyx_fixed}
Y =
\begin{cases}
0 & \text{if } x \in \gX_b,\\
\theta_i & \text{if } x \in \gX_i.
\end{cases}
\end{equation}
Let $P = P_X \times P_{Y|X}$. This distribution is realizable by $h^*(x)=0$ on $\gX_b$ and $h^*(x)=\theta_i$ on $\gX_i$, hence $a(h^*)=0$.

\textbf{Step 2: Lower bound on $a_{\bar b}(h_S)$ and $\varepsilon_S$.}
Let $I_S = \{i \in [m]: \gX_i \cap S \neq \emptyset\}$ and $U_S = [m]\setminus I_S$. Then $|I_S| \le n$, hence $|U_S| \ge m-n \ge m/2$.

For any $i \in U_S$, the variable $\theta_i$ is independent of $S$, hence
\begin{equation}
\label{eq:indep_correct}
\mathbb{E}_{\theta}\!\left[\mathbf{I}\{h_S(x)\neq Y\} \mid x \in \gX_i, S \right] = \tfrac{1}{2}.
\end{equation}

Therefore,
\begin{eqnarray}
\mathbb{E}_{S,\theta}[a_{\bar b}(h_S)]
&=& \mathbb{E}_S \left[
\sum_{i=1}^m \frac{1}{m}
\mathbb{E}_{\theta}[\mathbf{I}\{h_S(x)\neq Y\} \mid x \in \gX_i, S]
\right] \\
&\ge& \mathbb{E}_S \left[
\sum_{i \in U_S} \frac{1}{m} \cdot \frac{1}{2}
\right] \\
&=& \mathbb{E}_S \left[\frac{|U_S|}{2m}\right] \\
&\ge& \frac{1}{4}.
\label{eq:abarb_fixed}
\end{eqnarray}

Since the expected value of the risk over the random choice of $\theta$ satisfies the lower bound, there must exist at least one specific realization of $\theta$ (and thus a specific realizable distribution $P$) for which the following holds: 
\begin{equation}
\label{eq:outside}
\E_{S}[a_{\bar b}(h_S)] \ge 1/4.
\end{equation}
Since $a_b(h_S) \le 1$ and $c \le 1$, we have $\max\{a_b(h_S), c\} \le 1$, which implies $\varepsilon_S \ge a_{\bar b}(h_S)$. As a result, we have  $\E_S[\varepsilon_S] \ge \E_{S}[a_{\bar b}(h_S)] \ge 1/4$. 

\textbf{Step 3: Bounding $a(h_S)$.} Note that
\begin{equation}
\label{eq:risk_fixed}
a(h_S) = p_b a_b(h_S) + (1-p_b) a_{\bar b}(h_S) \ge (1-p_b) a_{\bar b}(h_S).
\end{equation}

We next consider the following cases.

\emph{Case 1: $a_b(h_S) \ge c$.} Observe that $\varepsilon_S = \frac{a_{\bar b}(h_S)}{a_b(h_S)} \le \frac{a_{\bar b}(h_S)}{c} $. So $a_{\bar b}(h_S) \ge c \varepsilon_S$. 

\emph{Case 2: $a_b(h_S) < c$.} Then $\varepsilon_S = \frac{a_{\bar b}(h_S)}{c}$, and hence $a_{\bar b}(h_S) = c\varepsilon_S$.

Both cases show that $a_{\bar b}(h_S) \ge c\varepsilon_S$. Combining this with  \eqref{eq:risk_fixed} , we obtain $a(h_S) \ge c(1-p_b) \varepsilon_S$. This holds for every $\theta$. Taking expectation over $S$ yields 
\begin{equation}
\label{eq:ratio}
\E_{S}[a(h_S)] \ge c(1-p_b)  \E_{S}[\varepsilon_S],
\end{equation}
completing the proof.
\end{proof}

\begin{proof}[Proof of Corollary \ref{cor:rate-lower-bound}]
By Theorem~\ref{thm-No-free-fair-algorithm}, for any $\gA$,
\[
\sup_{P} \mathbb{E}[\varepsilon_S] \;\ge\; \frac{1}{4},
\qquad
\sup_{P} \frac{\mathbb{E}[a(\gA(S))]}{\mathbb{E}[\varepsilon_S]} \;\ge\; c(n)(1-p_b).
\]
Let $P_1,P_2\in\mathcal{P}_{\mathrm{real}}(p_b)$ be (near-)maximizers of the two suprema above. Then
$\mathbb{E}_{P_1}[\varepsilon_S]\ge 1/4$ and
$\mathbb{E}_{P_2}[a(\gA(S))]\ge c(n)(1-p_b)\,\mathbb{E}_{P_2}[\varepsilon_S]$.
Taking $P^\star$ to be whichever of $P_1,P_2$ yields the larger $\mathbb{E}[a(\gA(S))]$, we obtain
\[
\sup_{P} \mathbb{E}[a(\gA(S))]
\;\ge\; \mathbb{E}_{P^\star}[a(\gA(S))]
\;\ge\; \frac{1-p_b}{4}\,c(n).
\]
Taking $\inf_{\gA}$ completes the proof.
\end{proof}

\section{Proof for No-free-fairness Architecture} \label{app-Proofs-No-free-fairness-Architecture}

\begin{proof}[Proof of Theorem \ref{thm-No-free-fair-Architecture}]

First, we observe that $a_b(h) \ge a_b^* >0$, by definition. We next analyze two cases.

\textbf{Case 1:} $\varepsilon \le 1$. Observe that
\begin{eqnarray}
a(h) &=& p_b a_b(h) + (1-p_b) a_{\bar{b}}(h) \\
&=&  p_b a_b(h) + (1-p_b) \varepsilon a_b(h) \\
&=& (1-\varepsilon) p_b a_b(h) + \varepsilon a_b(h) \\
&\ge& \varepsilon a_b(h) \\
&\ge& \varepsilon a_b^*.
\end{eqnarray}

\textbf{Case 2:} $\varepsilon > 1$. Observe that
\begin{eqnarray}
a(h) &=& p_b a_b(h) + (1-p_b) a_{\bar{b}}(h) \\
&=&  p_b a_b(h) + (1-p_b) \varepsilon a_b(h) \\
&\ge& \varepsilon(1-p_b) a_b(h) \\
&\ge& \varepsilon(1-p_b) a_b^*.
\end{eqnarray}

Combining those two cases lead to the following:
\begin{eqnarray}
a(h) &\ge& 
\begin{cases}
	\varepsilon(1-p_b) a_b^* & \text{ if } \varepsilon >1\\
	\varepsilon a_b^* & \text{ otherwise } 
\end{cases} \\
&=& \varepsilon a_b^*
\begin{cases}
	1-p_b & \text{ if } \varepsilon >1\\
	1 & \text{ otherwise } 
\end{cases} \\
&=& \varepsilon a_b^* \big[1 - p_b \mathbf{I}\{\varepsilon>1\}\big].
\end{eqnarray}
\end{proof}

\section{Justification for the Risk Ratio as a Fairness Metric} \label{app:risk-ratio-justification}

Throughout this work, we formalize relative fairness via the \emph{risk ratio} $\mathrm{RR}(h)=a_{\bar b}(h)/a_b(h)$ rather than the mathematically friendlier absolute difference $\Delta(h)=|a_{\bar b}(h)-a_b(h)|$. Although $\Delta(h)$ admits standard polynomial convergence rates in realizable PAC regimes \cite{woodworth2017learning,agarwal2018reductions,laakom25FairnessOverfitting}, it is intrinsically \emph{scale-dependent} and can be statistically and operationally misleading. In contrast, ratio-based notions encode \emph{proportional harm}, which is the quantity of interest in many high-stakes settings. Below we articulate the theoretical and practical imperatives for adopting $\mathrm{RR}(h)$, thereby motivating why the statistical bottlenecks identified in our main results are fundamental.

\begin{itemize}
    \item \textbf{Scale sensitivity in low-risk regimes.}     Absolute difference ignores the baseline magnitude of risk. For instance, let's fix the same $\Delta(h)=0.01$. If $a_b(h)=0.51$ and $a_{\bar b}(h)=0.50$, then $\mathrm{RR}(h)\approx 1.02$, indicating negligible disparity. In contrast, if $a_b(h)=0.02$ and $a_{\bar b}(h)=0.01$, then $\mathrm{RR}(h)=2$, i.e., a $100\%$ multiplicative increase in risk. More generally, for $a_{\bar b}(h)\ge a_b(h)$ we have $\mathrm{RR}(h)=1+\Delta(h)/a_b(h)$, which diverges as $a_b(h)\to 0$. Hence, absolute difference $\Delta(h)$ systematically under-penalizes disparities in low-incidence regimes.

    \item \textbf{Sensitivity under multiplicative improvements.}    Suppose model updates induce multiplicative improvements $a_b(h')=\gamma_b a_b(h)$ and $a_{\bar b}(h')=\gamma_{\bar b} a_{\bar b}(h)$ with $\gamma_b,\gamma_{\bar b}\in(0,1)$. Then $\mathrm{RR}(h')/\mathrm{RR}(h)=\gamma_{\bar b}/\gamma_b$. Even when both groups improve, relative disparity worsens whenever $\gamma_{\bar b}>\gamma_b$. In contrast, $\Delta(h')=\gamma_{\bar b}a_{\bar b}(h)-\gamma_b a_b(h)$ may shrink purely due to global error reduction, masking a strict widening of proportional inequity.

    \item \textbf{Compatibility with ratio-based constraints in regulation.}    Many regulatory standards are explicitly ratio-based. For instance, the U.S. Equal Employment Opportunity Commission's (EEOC) ``Four-Fifths Rule'' requires the ratio of selection rates to be at least $0.8$. Such constraints are naturally expressed as lower bounds on a ratio. Optimizing $\Delta(h)$ provides no guarantee of satisfying constraints of the form $\mathrm{RR}(h)\le \tau$ or $\mathrm{RR}(h)\ge \tau$, whereas directly controlling $\mathrm{RR}(h)$ yields immediate compliance guarantees (cf.~\cite{feldman2015certifying,hardt2016equality}).
    
     \item \textbf{Alignment with epidemiological and risk-analysis standards.}   The ratio $\mathrm{RR}(h)$ coincides with \emph{relative risk}, the canonical measure in epidemiology for quantifying disproportionate harm. This ensures that fairness assessments in ML are directly interpretable within established scientific frameworks (cf.~\cite{pearl2014interpretation}).

    \item \textbf{Scale invariance and cross-domain interpretability.}    The quantity $\Delta(h)$ is tied to the absolute scale of the task and lacks a canonical interpretation across domains. In contrast, $\mathrm{RR}(h)$ is dimensionless: $\mathrm{RR}(h)=1.5$ uniformly encodes a $50\%$ excess risk regardless of whether $a_b(h)$ is $10^{-3}$ or $0.3$. This invariance is essential for transferring fairness thresholds across heterogeneous applications.

    \item \textbf{Statistical behavior under class imbalance.}    When $p_b\ll 1$, accurate estimation of $a_b(h)$ requires $n\gg 1/p_b$, and errors in estimating $a_b(h)$ can dominate $\Delta(h)$. For the ratio, the relevant quantity is multiplicative error, e.g., $|\widehat{a}_{\bar b}(h)/\widehat{a}_b(h)-a_{\bar b}(h)/a_b(h)|$, which directly reflects proportional uncertainty. Our lower bounds show that controlling such quantities necessarily incurs sample complexity scaling with $1/\min\{a_b(h),a_{\bar b}(h)\}$, revealing an intrinsic dependence on rare-event regimes.

    \item \textbf{Incompatibility of absolute control with proportional guarantees.}  Without assumptions on $a_b(h)$,   there is no distribution-free mapping from a bound $\Delta(h)\le \nu$ to a meaningful bound on $\mathrm{RR}(h)$. Indeed, $\mathrm{RR}(h)\le 1+\nu/a_b(h)$, which becomes vacuous as $a_b(h)\to 0$. This means prior results on absolute disparity do not directly apply to ratio-based fairness. Therefore, guaranteeing proportional fairness necessitates direct control of ratio-type quantities, which in turn induces the sharper statistical trade-offs established in our main results.
\end{itemize}

In summary, absolute disparity is analytically convenient but misaligned with proportional notions of harm. In contrast, $\mathrm{RR}(h)$ is scale-invariant, regulation-compatible, and statistically faithful to rare-event regimes. Consequently, the stringent sample complexity requirements and trade-offs identified in our results are unavoidable: they arise from the intrinsic difficulty of estimating and controlling multiplicative disparities.



\end{document}